\documentclass[english,11pt, a4paper, final]{article}
\usepackage[T1]{fontenc}
\usepackage[latin9]{inputenc}
\usepackage{geometry}
\usepackage{color}
\usepackage{babel}
\usepackage{units}
\usepackage{mathtools}
\usepackage{enumitem}
\usepackage{amsmath}
\usepackage{amsthm}
\usepackage{amssymb}
\usepackage[numbers]{natbib}
\usepackage{microtype}
\usepackage[unicode=true,
 bookmarks=true,bookmarksnumbered=false,bookmarksopen=false,
 breaklinks=false,pdfborder={0 0 0},pdfborderstyle={},backref=page,colorlinks=true]
 {hyperref}
\hypersetup{pdftitle={Kernels},
 pdfauthor={Paul Dommel},
 citecolor=darkgray, urlcolor= darkgray, linkcolor= darkgray}

\makeatletter
\numberwithin{equation}{section}
\theoremstyle{plain}
\newtheorem{thm}{\protect\theoremname}[section]
\theoremstyle{remark}
\newtheorem{rem}[thm]{\protect\remarkname}
\theoremstyle{plain}
\newtheorem{prop}[thm]{\protect\propositionname}
\theoremstyle{plain}
\newtheorem{lem}[thm]{\protect\lemmaname}

\usepackage[notcite,notref]{showkeys}
\usepackage{dsfont}   
\usepackage{newtxtext, newtxmath}	
\usepackage{microtype}
\usepackage{tikz}
\usetikzlibrary{matrix,arrows.meta}
\usepackage{algorithm,algpseudocode}

\usepackage{doi}
\usepackage{orcidlink}

\setlist[enumerate,1]{label=(\roman*)}  

\DeclareMathOperator{\E}{{\mathds E}}

\DeclareMathOperator*{\Tr}{tr}

\newtheorem{assumption}{Assumption}

\makeatother

\providecommand{\lemmaname}{Lemma}
\providecommand{\propositionname}{Proposition}
\providecommand{\remarkname}{Remark}
\providecommand{\theoremname}{Theorem}

\begin{document}
\title{A Bound on the Maximal Marginal Degrees of Freedom}
\author{Paul Dommel\thanks{University of Technology, Chemnitz, Faculty of mathematics. 90126
Chemnitz, Germany}}
\maketitle
\begin{abstract}
Kernel ridge regression, in general, is expensive in memory allocation
and computation time. This paper addresses low rank approximations
and surrogates for kernel ridge regression, which bridge these difficulties.
The fundamental contribution of the paper is a lower bound on the
minimal rank such that the prediction power of the approximation remains
reliable. Based on this bound, we demonstrate that the computational
cost of the most popular low rank approach, which is the Nyström method,
is almost linear in the sample size. This justifies the method from
a theoretical point of view. Moreover, the paper provides a significant
extension of the feasible choices of the regularization parameter.

The result builds on a thorough theoretical analysis of the approximation
of elementary kernel functions by elements in the range of the associated
integral operator. We provide estimates of the approximation error
and characterize the behavior of the norm of the underlying weight
function. \medskip{}


\end{abstract}

\section{Introduction}

A large and important class of algorithms in machine learning involves
kernels. Due to the famous kernel trick, these methods are able to
reflect nonlinear relationships, which makes kernels attractive for
a variety of differing machine learning tasks. These tasks include
classification and regression, and they are employed to reduce the
dimension of the problem under consideration as well. Most importantly,
kernel methods have proven to be valuable in different scientific
areas such as geostatistics (cf.\ \citet{Honarkhah2010}), stochastic
optimization (cf.\ \citet{DommelPichler2023}, \citet{Hanasusanto}),
digit recognition (cf.\ \citet{Scholkopf1997}), computer vision
(cf.\ \citet{ZhangComputerVision}) and bio informatics (cf.\ \citet{Schlkopf2005Biologie}).

The most popular approach is kernel ridge regression (KRR), which
aims to infer a response variable~$Y$ based on an observed variable~$X$.
To this end, KRR approximates the conditional expectation $f_{0}(x)=\E\left(\left.Y\right|X=x\right)$
by a linear combination $\sum_{i=1}^{n}w_{i}k(\cdot,x_{i})$ of kernels
located at the sampling points. Despite excellent theoretical properties,
such as consistency and fast convergence (see, for example, \citet{Caponnetto2006},
\citet{eberts11_neurips} or \citet{dommel2021uniform}), the approach
is intractable in large scale scenarios as it suffers from its high
computational requirements. Indeed, its complexity heavily depends
on the sample size, which \textendash{} in the standard way \textendash{}
is of order $\mathcal{O}(n^{2})$ in memory and $\mathcal{O}(n^{3})$
in time, respectively. 

To overcome these computational difficulties, several techniques have
been developed. One class of methods combines iterative algorithms
(gradient descent or the conjugate gradient method (CG), e.g.)\ with
early stopping, which reduces the computation time of the regression
coefficients (cf.\ \citet{Raskutti2011}, \citet{Rosasco2007}).
Other methods are based on low rank approximations of the kernel matrix
(cf.\ \citet{Fine2002}) or the kernel function itself (see \citet{Rahimi2007RandomFF}).
The \emph{Falcon algorithm}, presented in \citet{Rudi2017FALKONAO},
combines both ideas, a low rank formulation of the regression problem
with a preconditioned CG approach. Another notable method is the \emph{divide
and conquer} approach presented in \citet{Zhang2013}. The method
divides the data set into subsets, determines the corresponding kernel
ridge regressions and averages them afterwards. Given some conditions,
which are generally hard to verify, both methods achieve the same
minimax convergence rate as kernel ridge regression, while substantially
reducing the computational effort.

One of the most famous techniques to improve the numerical scheme
is the \emph{Nyström method}, which was introduced in \citet{WlliamsSeeger2000}.
Due to its simplicity and its computational benefits it is employed
frequently in a variety of different areas (see, for example, \citet{Sriperumbudur2021}
and \citet{altschuler2019massively}). The key idea of the Nyström
method is to determine the regression weights based a low rank approximation
of the kernel matrix $(k(x_{i},x_{j}))_{i,j=1}^{n}$. This approximation
involves $p<n$ centers chosen from the sampling points and it allows
computing the regression weights in time of order $\mathcal{O}(np^{2})$.
If $p$ is small, this is a significant improvement compared to $\mathcal{O}(n^{3})$
of the standard approach. 

The inspiring work \citet{NIPS2015_03e0704b} analyzes the Nyström
method from a theoretical point of view and provides statistical guarantees
for its predictive performance. The authors show that the Nyström
method admits similar convergence rates as kernel ridge regression,
provided that the amount of centers~$p$ exceeds a specific threshold.
This threshold depends linearly on the \emph{maximal marginal degrees
of freedom,
\begin{equation}
\mathcal{N}_{\infty}(\lambda)\coloneqq\sup_{x\in\mathcal{X}}\left\langle k_{x},(\lambda+L_{k})^{-1}k_{x}\right\rangle _{k}.\label{eq:1}
\end{equation}
}Here, $L_{k}$ denotes the integral operator 
\begin{equation}
(L_{k}w)(y)\coloneqq\int_{\mathcal{X}}w(x)k(x,y)p(x)dx,\label{eq:intop}
\end{equation}
$\mathcal{X}\subset\mathbb{R}^{d}$ is the domain of the sampling
points, and $k_{x}\coloneqq k(\cdot,x)$.

The main drawback of this important theoretical threshold is that
the magnitude $\mathcal{N}_{\infty}(\lambda)$ and its behavior in
$\lambda$ are generally unknown. In the worst case scenario, $\mathcal{N}_{\infty}(\lambda)$
might grow proportionally to $\lambda^{-1}$, which, for the usual
regularization choice $\lambda=cn^{-1}$, implies that the lower bound
of centers has to be at least linear in the sample size $n$. In other
words, the computational complexity of the Nyström method grows as
fast as the complexity of kernel ridge regression, which is not an
improvement and contradicts the purpose of the method.

This paper addresses this issue by exploring the asymptotic behavior
of the maximal marginal degrees of freedom for a specific class of
kernel functions. The main result is an upper bound on $\mathcal{N}_{\infty}(\lambda)$,
which grows significantly more slowly than $\lambda^{-1}$, provided
that the kernel is sufficiently smooth. This justifies the Nyström
method from a theoretical point of view. Our approach relies on the
crucial identity
\begin{equation}
\left\langle k_{x},(\lambda+L_{k})^{-1}k_{x}\right\rangle _{k}=\min_{w\in L^{2}(P)}\lambda n\left\Vert w\right\Vert _{L^{2}(P)}^{2}+\left\Vert L_{k}w-k_{x}\right\Vert _{k}^{2},\label{eq:Intro}
\end{equation}
which relates $\mathcal{N_{\infty}}(\lambda)$ with a convex optimization
problem. We study the latter by employing a suitably chosen weight
function~$w$, based on which we establish an upper bound on the
objective value. This bound is uniform in the location $x\in\mathcal{X}$
and therefore sufficient for the desired quantity~\eqref{eq:1}.
Building on this crucial result, we then refine two major results
from \citet{NIPS2015_03e0704b}. First, we provide an explicit lower
bound on the necessary number of Nyström centers, demonstrating that
this method is asymptotically cheaper than kernel ridge regression,
while maintaining the same convergence rate. The second main implication
addresses the regularization for general kernel methods. The result
extends common regularization choices, which are inversely proportional
to the sample size, towards significantly smaller regularization parameters.
It connects the discrete learning setting with the continuous framework,
which constitutes the foundation for the theoretical analysis of kernel
methods. 

\paragraph{Outline of the paper. }

Section~\ref{sec:2} addresses reproducing kernel Hilbert spaces,
recalls their relevant properties and introduces the notation. The
subsequent Section~\ref{sec:3} presents the main result of the paper.
That is, the bound on the maximal marginal degrees of freedom, based
on which we conclude an explicit bound for a sufficient amount of
Nyström centers as well as proper regularization choices. Section~\ref{sec:4}
summarizes the results, while the final Section~\ref{sec:5} provides
the proof of the main result. 

\section{\label{sec:2} Preliminaries and notation}

\emph{Reproducing kernel Hilbert spaces} (RKHS) constitute an exclusive
class of Hilbert spaces, which are considered in a variety of different
mathematical fields. Due to the famous kernel trick (cf.\ \citet[p.~292]{Bishop2006}),
RKHS are of particular interest in the field of mathematical learning
theory. In what follows, we introduce these spaces as well as notations,
which we use throughout this paper.

We consider a symmetric, positive definite kernel $k\colon\mathcal{X}\times\mathcal{X}\to\mathbb{R}$
defined on a compact set $\mathcal{X}$, which we refer to as \emph{design
space}. On the linear span $H_0 \coloneqq \operatorname{span} \big \{ k(\cdot , x) \colon x \in \mathcal{X} \big \}$
of differently located kernel functions, we define the inner product
\[
\left\langle \sum_{i=1}^{n}\alpha_{i}k(\cdot,x_{i}),\sum_{j=1}^{m}\beta_{j}k(\cdot,y_{j})\right\rangle _{k}\coloneqq\sum_{i,j=1}^{n,m}\alpha_{i}\beta_{j}k(x_{i},y_{i}),
\]
as well as the corresponding norm $\left\Vert f\right\Vert _{k}=\sqrt{\left\langle f,f\right\rangle _{k}}$.
The associated reproducing kernel Hilbert space $\mathcal{H}_{k}$
is the closure of $H_{0}$ with respect to the norm $\left\Vert \cdot\right\Vert _{k}$.
The name building property of $\mathcal{H}_{k}$ is the \emph{reproducing
property} 
\begin{equation}
\left\langle k(\cdot,x),f(\cdot)\right\rangle _{k}=f(x),\hspace{1em}f\in\mathcal{H}_{k},\label{eq:Reproducing Prop}
\end{equation}
which holds true for every $x\in\mathcal{X}$. In other words, a point
evaluation at $x\in\mathcal{X}$ is an inner product with the kernel
function $k_{x}(\cdot)\coloneqq k(\cdot,x)$. For this reason we refer
to $k_{x}$ as \emph{point evaluation function} at $x$. 

Closely linked to the kernel $k$ and occurring naturally in investigating
RKHS spaces is the operator $L_{k}\colon L^{2}(P)\to\mathcal{H}_{k}$
with 
\[
(L_{k}w)(y)=\int_{\mathcal{X}}w(x)k(y,x)p(x)dx
\]
and $w\in L^{2}(P)$. Here, $p\colon\mathcal{X}\to[0,\infty)$ is
a strictly positive density of the probability measure $P$, which
is called \emph{design measure. }$L_{k}$ is a self adjoint and positive
definite operator and therefore possesses a root $L_{k}^{\nicefrac{1}{2}}$
such that $L_{k}^{\nicefrac{1}{2}}\circ L_{k}^{\nicefrac{1}{2}}=L_{k}$.
This root relates the Lebesgue space $L^{2}(P)$ with $\mathcal{H}_{k}$
in the sense that $\mathcal{H}_{k}=$\emph{ }$L_{k}^{\nicefrac{1}{2}}(L^{2}(P))$
as well as $\left\Vert f\right\Vert _{k}=\left\Vert w\right\Vert _{L^{2}(P)}$,
where $f=L_{k}^{\nicefrac{1}{2}}w$. Similarly, it holds for the inner
product of $f,g\in\mathcal{H}_{k}$ that 
\begin{equation}
\left\langle f,g\right\rangle _{k}=\left\langle w,v\right\rangle _{L^{2}},\label{eq:normid}
\end{equation}
where $f=L_{k}^{\nicefrac{1}{2}}w$ and $g=L_{k}^{\nicefrac{1}{2}}v$.
In this work, $L_{k}$ constitutes the building block for analyzing
the asymptotic behavior of $\mathcal{N}_{\infty}(\lambda)$ in \eqref{eq:1}.
This quantity should not be confused with the \emph{effective dimension}
$\mathcal{N}(\lambda)\coloneqq\Tr((\lambda+L_{k})^{-1}L_{k})$, which
also frequently appears in the context of kernel ridge regression.
Both are connected through $\mathcal{N}_{x}(\lambda)=\left\langle k_{x},(\lambda+L_{k})^{-1}k_{x}\right\rangle _{k}$
as $\mathcal{N}(\lambda)=\E_{x}\mathcal{N}_{x}(\lambda)$ and $\mathcal{N}_{\infty}(\lambda)=\sup_{x\in\mathcal{X}}\mathcal{N}_{x}(\lambda)$
if $x\sim P$ (cf. \citet[p. 4]{NIPS2015_03e0704b}). 

In what follows, we utilize the notations introduced above to state
the main results in the subsequent section. 

\section{\label{sec:3} Main results}

The main result of this paper is an asymptotic bound on the (continuous)
maximal marginal degrees of freedom $\mathcal{N}_{\infty}(\lambda)$.
To this end, we consider radial kernel functions of the form 
\begin{equation}
k(x,y)=\phi\left(d^{-1}\left\Vert x-y\right\Vert _{2}^{2}\right),\label{eq:KernelShape}
\end{equation}
where $\phi$ is a smooth function on the unit interval $[0,1]$.
This class includes a variety of different kernels, which are of high
importance in practical applications. Among others, the Gaussian kernel,
the inverse multiquadric kernel as well as the class of Matern kernels
have a representation as in~\eqref{eq:KernelShape}. 
\begin{rem}
\label{rem:Term} The term maximal marginal degrees of freedom originates
from \citet{Bach2013}, where it was used for the maximal leverage
score of the kernel matrix. That is, 
\[
d_{n}(\lambda)\coloneqq\max_{i=1,\dots,n}n\left(K\left(K+n\lambda I\right)^{-1}\right)_{ii},
\]
where $K=(k(x_{i},x_{j}))_{i,j=1}^{n}$ denotes the kernel matrix
corresponding to the data set $\{x_{1},\dots,x_{n}\}$. This quantity
has the equivalent representation 
\begin{equation}
\max_{i=1,\dots,n}nK\left(K+n\lambda I\right)_{ii}^{-1}=\max_{i=1,\dots,n}\left\langle \left(\lambda+L_{k}^{D}\right)^{-1}k_{x_{i}},k_{x_{i}}\right\rangle _{k},\label{eq:MMDF}
\end{equation}
where $L_{k}^{D}\colon\mathcal{H}_{k}\to\mathcal{H}_{k}$ with $(L_{k}^{D}f)(y)=\frac{1}{n}\sum_{i=1}^{n}f(x_{i})k(y,x_{i})$
is the discrete version of the integral operator $L_{k}$. For $n\to\infty$,
the operator $L_{k}^{D}$ converges towards the integral operator
$L_{k}$ and, likewise, $d_{n}(\lambda)$ towards $\mathcal{N}_{\infty}(\lambda)$
in \eqref{eq:MMDF}. In other words, we consider a continuous version
of the original maximal marginal degrees of freedom and therefore
keep the terminology. 
\end{rem}

The main obstacle for analyzing the asymptotics of $\mathcal{N}_{\infty}(\lambda)$
is that the eigenvalues and eigenfunctions of $L_{k}$ are generally
not accessible. It is therefore necessary to involve a different characterization
of the maximal marginal degrees of freedom. That is, the max min characterization
\begin{equation}
\mathcal{N}_{\infty}(\lambda)=\sup_{x\in\mathcal{X}}\inf_{w\in L^{2}}\left\Vert w\right\Vert _{L^{2}}^{2}+\lambda^{-1}\left\Vert L_{k}w-k_{x}\right\Vert _{k}^{2},\label{eq:centralid}
\end{equation}
which follows, as $w_{\lambda}^{x}=(\lambda+L_{k})^{-1}k_{x}$ minimizes
the inner objective on the right hand side (cf.\ Lemma~\ref{lem:centralid}).
With this we may assess the asymptotic behavior of $\mathcal{N}_{\infty}(\lambda)$
without concrete knowledge of the spectral quantities of $L_{k}$.
To this end, we study the inner objective for a specific sequence
of weight functions~$w$ and derive a bound on its optimal value.
This sequence is chosen such that the norm of the weight functions
grows far more slowly than the approximation error $\left\Vert L_{k}w-k_{x}\right\Vert _{k}^{2}$
decays. The decay rate of this error primarily depends on smoothness
characteristics of the underlying kernel function~\eqref{eq:KernelShape},
which are determined by the outer function $\phi$. For the affiliated
function $\tilde{\phi}\colon[-1,1]\to\mathbb{R}$ with $\tilde{\phi}(x)=\phi\big(\frac{x+1}{2}\big)$,
we distinguish two general cases, which are specified by the following
assumptions from \citet[Theorem 2.2]{Wang2018} and \citet[Lemma 2]{Wang2021}:

\begin{assumption}\label{asuA}The functions $\tilde{\phi},\tilde{\phi}^{\prime},\dots,\tilde{\phi}^{(s-1)}$
are absolutely continuous with $s>3d$. Furthermore, $\tilde{\phi}^{(s)}$
is of bounded variation and 
\[
V_{s}\coloneqq\int_{-1}^{1}\tilde{\phi}^{(s+1)}(x)(1-x^{2})^{-\frac{1}{4}}<\infty.
\]

\end{assumption}

\begin{assumption}\label{asuB}The function $\ensuremath{\tilde{\phi}}$
is analytic in the region bounded by the ellipse $\ensuremath{\mathcal{E}_{\rho}}$
for some $\rho>1$. \end{assumption}

In what follows, we state the main result of this paper. It reveals
the precise bound on the maximal marginal degrees of freedom based
on the assumptions stated above. 
\begin{thm}[Bound on $\mathcal{N}_{\infty}(\lambda)$]
\label{thm:Main} Let $\mathcal{X}=[0,1]^{d}$ be equipped with a
design measure $P$, such that its density $p$ satisfies 
\[
0<p_{\min}\coloneqq\inf_{z\in\mathcal{X}}p(z)\le p_{\max}\coloneqq\sup_{z\in\mathcal{X}}p(z)<\infty.
\]
Further, let $k$ be a kernel of the form \eqref{eq:KernelShape}.
Then, provided $\lambda$ is sufficiently small, the maximal marginal
degrees of freedom fulfill the bound
\begin{equation}
\mathcal{N}_{\infty}(\lambda)\le C_{s}\lambda^{-\frac{2d}{s-d-\frac{1}{2}}},\label{eq:polbound}
\end{equation}
if Assumption~\ref{asuA} is satisfied. Further, it holds that 
\begin{equation}
\mathcal{N}_{\infty}(\lambda)\le C_{\rho}\ln(\lambda^{-1})^{2d},\label{eq:expbound}
\end{equation}
if Assumption~\ref{asuB} is fulfilled. $C_{s}$ and $C_{\rho}$
are constants depending on the density $p$, the dimension $d$ and
the outer function $\phi$.
\end{thm}

Section~\ref{sec:5} below provides the proof of this theorem and
the following results as well.
\begin{rem}
The results may be extended to more general rectangular domains of
the shape $\mathcal{X}=[a_{1},b_{1}]\times\dots\times[a_{d},b_{d}]$,
where $a_{i}<b_{i}$ for all $i=1,\dots,n$. This setting requires
an adjusted approach, which we outline in the Remarks \ref{rem:Gen1}
and \ref{rem:Gen2}. 
\end{rem}

\begin{rem}
For a data set $(x_{1},y_{1}),\dots,(x_{n},y_{n})$ one may consider
the discrete version
\begin{equation}
\max_{x\in\mathcal{X}}\min_{w\in\mathbb{R}^{n}}\lambda\left\Vert w\right\Vert _{2}^{2}+\left\Vert \sum_{i=1}^{n}w_{i}k(\cdot,x_{i})-k_{x}\right\Vert _{k}^{2}\label{eq:discrcentralobj}
\end{equation}
of the optimization problem \eqref{eq:centralid}. This problem has
a nice interpretation in the context of kernel ridge regression. It
is easy to see that kernel ridge regression predicts the function
value at a new point $x$ by $w_{\lambda}^{x^{\top}}y$, where $y=(y_{1},\dots,y_{n})$
is the vector containing the observations and $w_{\lambda}^{x^{\top}}$
the solution of~\eqref{eq:discrcentralobj}. Moreover, the error
of the prediction has the decomposition 
\[
\left|f_{0}(x)-w_{\lambda}^{x\top}y\right|\le\left|\sum_{i=1}^{n}(f_{0}(x_{i})-y_{i})w_{\lambda,i}^{x}\right|+\left|\left\langle f_{0},k_{x}-\sum_{i=1}^{n}w_{\lambda,i}^{x}k(\cdot,x_{i})\right\rangle _{k}\right|,
\]
where $f_{0}$ denotes the underlying regression function. Hence,
the error $\left\Vert \sum_{i=1}^{n}w_{i}k(\cdot,x_{i})-k_{x}\right\Vert _{k}^{2}$
controls the accuracy of the prediction at a new point $x\in\mathcal{X}$,
while the regularization term $\lambda\left\Vert w\right\Vert _{2}^{2}$
controls the influence of the noise $(y_{i}-f_{0}(x_{i}))$ for $i=1,\dots,n$.
In other words, the optimization problem \eqref{eq:discrcentralobj}
considers the trade-off between an accurate and a stable prediction
for the worst case in the design space $\mathcal{X}$. The regularization
parameter determines the leverage of each of the single objectives. 

The bounds on $\mathcal{N}_{\infty}(\lambda)$ justify a variety of
low rank kernel methods from a theoretical point of view. These methods
rely on an approximation of the kernel matrix, which is computationally
more efficient than utilizing the true kernel matrix. Indeed, employing
low rank approximations is superior in both, memory allocation and
computation time. The most common representative is the Nyström-approximation,
based on which several methods have been developed. This includes
the Nyström method itself (cf.\ \citet{WlliamsSeeger2000}), Nyström
kernel PCA (cf.\ \citet{Sriperumbudur2021}) or the Nyström Sinkhorn
algorithm (cf.\ \citet{altschuler2019massively}). The state of the
art result for the Nyström method (cf.\ \citet{NIPS2015_03e0704b})
demands the rank of the approximation to be at least linear in $\mathcal{N}_{\infty}(\lambda)$,
for which Theorem~\ref{thm:Main} provides an asymptotic bound. 

With that, we now refine the result based on the assumptions in \citet{NIPS2015_03e0704b}
and provide an asymptotic lower bound on the required rank as follows. 
\end{rem}

\begin{prop}[Nyström method]
 Let $\mathcal{E}(f)\coloneqq\iint_{\mathcal{X}\times\mathbb{R}}\bigl(f(x)-y\bigr)^{2}\rho(dx,dy)$
be the common error function and assume that the assumptions of Theorem~1
in \citep{NIPS2015_03e0704b} and Theorem~\ref{thm:Main} are fulfilled.
With probability at least $1-\delta$, the Nyström-approximation $\hat{f}_{\lambda,m}$
satisfies 
\begin{equation}
\mathcal{E}(\hat{f}_{\lambda,m})-\mathcal{E}(f_{\mathcal{H}})\le q^{2}n^{\frac{2\nu+1}{2nu+\gamma+1}}\label{eq:Nystr=0000F6mBound}
\end{equation}
for at least
\begin{equation}
m\ge\left(67\lor5C_{s}\lambda^{-\frac{2d}{s-d-\frac{1}{2}}}\right)\ln\frac{12\kappa^{2}}{\lambda\delta}\label{eq:Nystr=0000F6mPol}
\end{equation}
supporting points, if~\ref{asuA} is fulfilled. Moreover, if~\ref{asuB}
is fulfilled, the assertion holds for
\begin{equation}
m\ge\left(67\lor5C_{\rho}\ln^{2d}(\lambda^{-1})\right)\ln\frac{12\kappa^{2}}{\lambda\delta}\label{eq:Nystr=0000F6mExp}
\end{equation}
supporting points. The reference provides $\lambda$ as well as the
constants $\kappa$, $q$ and $\nu$ explicitly.
\end{prop}

Building on the bound \eqref{eq:Nystr=0000F6mPol} for the amount
of Nyström centers, we may provide an explicit bound on its complexity
as well. In the standard way, the complexity of the Nyström method
is of order $\mathcal{O}(n\cdot m^{2})$ (cf.\ \citet{Bach2013}).
Assuming that the regularization parameter is chosen as $\lambda=cn^{-1}$
(which is the lower bound for $\lambda$ in \citep{NIPS2015_03e0704b}),
the amount of required samples is 
\[
m\ge\left(67\lor5C_{s}c^{-\frac{2d}{s-d-\frac{1}{2}}}n^{\frac{2d}{s-d}}\right)\ln\frac{12\kappa^{2}n}{c\delta}.
\]
With that, we find the order of complexity 
\[
\mathcal{O}(n\cdot m^{2})=\mathcal{O}\bigl(n^{1+\frac{4d}{s-d-\frac{1}{2}}}\ln(n)\bigr)
\]
in terms of the sample size $n$. For smooth kernels, where $s$ is
very large, the latter term $n^{\frac{4d}{s-d}}\ln(n)$ grows very
slowly and the method performs in almost linear time. 

Besides its application for the Nyström method, Theorem~\eqref{thm:Main}
is of interest for proper choices of the regularization parameter
as well. A typical requirement on $\lambda$ is that 
\begin{equation}
\left\Vert (\lambda+L_{k})^{-\nicefrac{1}{2}}(L_{k}-L_{k}^{D})(\lambda+L_{k})^{-\nicefrac{1}{2}}\right\Vert \label{eq:opnorm}
\end{equation}
is (with high probability) bounded by some constant, which is independent
of $n$ and $\lambda$ (see, for example, \citet{steinwartfischer2020}
and \citet{Caponnetto2006}). Here, $L_{k}^{D}$ denotes the operator
defined in Remark~\ref{rem:Term}. It turns out that this is the
case, whenever $\mathcal{N}_{\infty}(\lambda)$ is proportional to
$n$ (cf.\ \citet[Proposition 8]{NIPS2015_03e0704b}). Connecting
this characterization with the bounds~\eqref{eq:polbound} and~\eqref{eq:expbound}
yields the following result. 
\begin{prop}
\label{prop:Regularization} Let $x_{1}\dots,x_{n}\sim P$ be independent
identically distributed samples and $L_{k}^{D}$ as in Remark~\ref{rem:Term}
and suppose the assumptions in Theorem~\ref{thm:Main} are satisfied.
Then, for any $\delta>0$, it holds with probability $1-2\delta$
that
\begin{equation}
\left\Vert (\lambda+L_{k})^{-\nicefrac{1}{2}}(L_{k}-L_{k}^{D})(\lambda+L_{k})^{-\nicefrac{1}{2}}\right\Vert \le\frac{2\beta(\lambda)(1+C_{s}\lambda^{-\frac{2d}{s-d-\frac{1}{2}}})}{3n}+\sqrt{\frac{2\beta(\lambda)(1+C_{s}\lambda^{-\frac{2d}{s-d}})}{n}},\label{eq:regpol}
\end{equation}
if Assumption~\ref{asuA} is satisfied, and 
\begin{equation}
\left\Vert (\lambda+L_{k})^{-\nicefrac{1}{2}}(L_{k}-L_{k}^{D})(\lambda+L_{k})^{-\nicefrac{1}{2}}\right\Vert \le\frac{2\beta(\lambda)(1+C_{\rho}\ln(\lambda^{-1})^{2d})}{3n}+\sqrt{\frac{2\beta(\lambda)(1+C_{\rho}\ln(\lambda^{-1})^{2d}}{n}},\label{eq:reganalyt}
\end{equation}
if Assumption~\ref{asuB} holds true. Here, $\beta(\lambda)\coloneqq\ln\big(\frac{4\Tr(L_{k})}{\lambda\delta}\big)$.
\end{prop}

\begin{rem}
Common concentration results as in \citet{Caponnetto2006}, \citet{eberts11_neurips}
and \citet{steinwartfischer2020} require the underlying regularization
series to decay not faster than $\mathcal{O}(n^{-1})$. This is too
restrictive for many tasks including the analysis of low rank kernel
methods. The new bound in Proposition~\ref{prop:Regularization}
allows more general regularization schemes beyond $\mathcal{O}(n^{-1})$.
For example, provided the kernel satisfies Assumption~\ref{asuB},
the right hand side \eqref{eq:reganalyt} tends to 0 for every regularization
sequence decaying not faster than $\exp(n^{-\frac{1}{2d+1}})$. This
is a significant extension of previously feasible choices. 
\end{rem}

\section{\label{sec:4} Summary and future research}

Low rank approximations of kernel methods are an important tool to
reduce the computation time and memory allocation for large scale
learning problems. The Nyström method has been proposed in the literature
to overcome these difficulties. Its required rank is linear in the
maximal marginal degrees of freedom.

The paper provides an explicit bound on the maximal marginal degrees
of freedom, which mainly involves the regularity of the underlying
kernel. This result has a variety of crucial implications. Most importantly,
it is the foundation for analyzing the asymptotic complexity of the
standard Nyström method, while maintaining a stable prediction. We
demonstrate that this complexity is almost linear in the sample size,
provided that the kernel is sufficiently smooth. Even more, the approach
justifies regularization sequences, which are significantly smaller
than established in the literature before. 

A possible direction for future research is the extension of the underlying
approach towards more broad settings. It is based on scaled Legendre
polynomials and works well on cubes (and hyperrectangles), but does
not plainly extend to more general domains. 

\section{\label{sec:5}Proofs}

In this section we provide the proof of Theorem~\ref{thm:Main},
which includes several different steps. The central idea is analyzing
the optimization problem on the right hand side of \eqref{eq:Intro}
and then inferring the bound on $\mathcal{N}_{\infty}(\lambda)$. 

To this end, we construct a specific sequence $(w_{m}^{x})_{m=0}^{\infty}$
of functions and study the associated objective value. These functions
are a linear combination of scaled Legendre polynomials with coefficients
depending on the location $x\in\mathcal{X}$. We demonstrate that
$w_{m}^{x}$ satisfies a set of $m$ moment conditions and provide
a bound on its $L^{2}$\textendash norm, which is uniform in $x\in\mathcal{X}$.
Then, we utilize the specific structure \eqref{eq:KernelShape} of
the kernel as well as the moment conditions, to verify that the approximation
error $\left\Vert L_{k}w_{m}^{x}-k_{x}\right\Vert $ is bounded by
the Legendre series remainder of the outer function $\phi$. We further
connect imposed the smoothness conditions with the decay of the series
reminder to derive an explicit bound on the approximation error. Building
on this result, we provide a bound on the objective value from which
we finally conclude the bound on the maximal marginal degrees of freedom.

\subsection{The moment function\label{subsec:The-moment-function}}

In this section we introduce the central weight function $w_{m}^{x}$,
for which the image $L_{k}w_{m}^{x}$ is a suitable approximation
of the point evaluation function $k_{x}$. Its characterizing properties
are the moment conditions 
\begin{equation}
\int_{0}^{1}w_{m}^{x}(z)z^{k}dz=x^{k},\label{eq:MomProp}
\end{equation}
which hold true for all $x\in[0,1]$ and for every $k=0,\dots,m-1$.
The construction of this function involves the\emph{ scaled (and shifted)
Legendre} polynomials
\begin{equation}
Q_{k}(x)\coloneqq\sqrt{2k+1}P_{k}(2x-1)\hspace{1em}k=0,1,2,\dots,\label{eq:LegendreDerivativeRep}
\end{equation}
where $P_{k}$ are the standard Legendre polynomials (cf. \citet[p. 774]{abramowitz+stegun}).
We list some major features of these polynomials first and derive
$w_{m}^{x}$ subsequently.
\begin{prop}[Properties $Q_{k}$]
\label{prop:LegendreProp} The following statements holds true:
\end{prop}

\begin{enumerate}
\item \label{enu:OrthonormalShif}$\left\{ Q_{k}\right\} _{k=0}^{\infty}$
builds an orthonormal system on $[0,1]$, i.e., 
\begin{equation}
\int_{0}^{1}Q_{i}(x)Q_{j}(x)dx=\delta_{i,j}\label{eq:OrthoPoly}
\end{equation}
for $i,j=0,1,2,\dots$.
\item \label{enu:LegBound} It holds $|Q_{k}(x)|\le\sqrt{2k+1}$ for every
$x\in[0,1]$ and $k=0,1,2,\dots$.
\end{enumerate}
\begin{proof}
The assertion follows from \citet[p. 94]{Shen2011} as well as the
identity \eqref{eq:LegendreDerivativeRep}. 
\end{proof}
Building on the scaled Legendre polynomials, we now define the function,
which is crucial for the approximation of the point evaluation function
$k_{x}$. For some integer $m\in\mathbb{N}$, consider $w_{m}^{x}\colon[0,1]\to\mathbb{R}$
given by

\begin{equation}
w_{m}^{x}(z)=\sum_{\ell=0}^{m-1}Q_{\ell}(x)Q_{\ell}(z),\label{eq:wmx}
\end{equation}
where $x$ is a fixed point in $[0,1]$ and $Q_{\ell}$ from \eqref{eq:LegendreDerivativeRep}.
As $\{Q_{\ell}\}_{\ell=0}^{\infty}$ is a orthonormal basis on $[0,1]$,
it is evident that $(L_{k}w_{m}^{x})(y)\to k_{x}(y)$ for every $y\in[0,1]$
as $m\to\infty$. To determine the the precise approximation quality
the following property of $w_{m}^{x}$ crucial. 
\begin{prop}
The function $w_{m}^{x}$ in \eqref{eq:wmx} has the moment property
\eqref{eq:MomProp} for every $x\in[0,1]$.
\end{prop}

\begin{proof}
$\{Q_{\ell}\}_{\ell=0}^{\infty}$ builds an orthonormal basis on $[0,1]$
and thus every monomial has a unique representation 
\[
x^{k}=\sum_{\ell=0}^{\infty}a_{\ell}Q_{\ell}(x)=\sum_{\ell=0}^{\infty}\left\langle z^{k},Q_{\ell}(z)\right\rangle _{L^{2}[0,1]}Q_{\ell}(x)
\]
in terms of the scaled Legendre polynomials. Furthermore, every coefficient
$a_{\ell}$ with $\ell>k$ vanishes as $x^{k}$ is contained in the
linear span of $\{Q_{0},Q_{1},\dots,Q_{k}\}$. Thus 
\[
x^{k}=\sum_{\ell=0}^{k}\left\langle z^{k},Q_{\ell}(z)\right\rangle _{L^{2}[0,1]}Q_{\ell}(x)
\]
holds for every $x\in[0,1]$. It therefore follows that 
\begin{equation}
\int_{0}^{1}w_{m}^{x}(z)z^{k}dz=\sum_{\ell=0}^{m-1}Q_{\ell}(x)\int_{0}^{1}Q_{\ell}(z)z^{k}dz=\sum_{\ell=0}^{k}\left\langle z^{k},Q_{\ell}(z)\right\rangle _{L^{2}[0,1]}Q_{\ell}(x)=x^{k}\label{eq:argument}
\end{equation}
for every $k\le m-1$. This is the assertion.
\end{proof}
\begin{rem}
\label{rem:Gen1}The moment property may be extended towards general
intervals with the following linear transformation. Given an interval
$[a,b]$ as well as a point $x\in[a,b]$, we define the modified moment
function $w_{m,[a,b]}^{x}(z)\coloneqq(b-a)^{-1}w_{m}^{\tilde{x}}\big(\frac{z-a}{b-a}\big)$,
where $\tilde{x}\coloneqq(b-a)^{-1}x-a(b-a)^{-1}$. Then, it follows
that 
\begin{align*}
\int_{a}^{b}z^{k}w_{m,[a,b]}^{x}(z)dz=(b-a)^{-1}\int_{a}^{b}z^{k}w_{m}^{\tilde{x}}\left(\frac{z-a}{b-a}\right)dz & =\int_{0}^{1}(y(b-a)+a)^{k}w_{m}^{\tilde{x}}(y)dy
\end{align*}
by substituting $y=\frac{z-a}{b-a}$. Further, employing the regular
moment property \eqref{eq:MomProp}, it holds
\begin{align*}
\int_{0}^{1}(y(b-a)+a)^{k}w_{m}^{\tilde{x}}(y)dy & =\sum_{\ell=0}^{k}\binom{k}{\ell}\int_{0}^{1}(y(b-a))^{\ell}a^{\ell-k}w_{m}^{\tilde{x}}(y)dy\\
 & =\sum_{\ell=0}^{k}\binom{k}{\ell}(\tilde{x}(b-a))^{\ell}a^{\ell-k}=(\tilde{x}(b-a)+a)^{k}=x^{k},
\end{align*}
which is the moment property \ref{eq:MomProp} for the Lebesgue measure
on $[a,b]$.
\end{rem}

Besides the approximation error $\left\Vert L_{k}w-k_{x}\right\Vert _{k}^{2}$
, the central objective 
\begin{equation}
\inf_{w\in L^{2}}\lambda\left\Vert w\right\Vert _{L^{2}}^{2}+\left\Vert L_{k}w-k_{x}\right\Vert _{k}^{2}\label{eq:centralobj}
\end{equation}
involves the norm of the underlying weight function as well. We therefore
provide a bound on the norm $\left\Vert w_{m}^{x}\right\Vert _{L^{2}[0,1]}$
of the auxiliary function $w_{m}^{x}$ next. This bound is uniform
in the location of $x\in[0,1]$, which is necessary to derive a uniform
bound on \eqref{eq:centralobj} later on. 
\begin{prop}[Norm of the minimal moment function]
 The bound
\begin{equation}
\sup_{x\in\mathcal{X}}\left\Vert w_{m}^{x}\right\Vert _{L^{2}[0,1]}^{2}\le m^{2},\label{eq:momfuncnorm}
\end{equation}
holds for every $m\in\mathbb{N}$. 
\end{prop}

\begin{proof}
The scaled Legendre polynomials build an orthonormal basis and they
are bounded by $\left|Q_{\ell}(x)\right|\le\sqrt{2\ell+1}$, where
$x\in[0,1]$. Thus, we have that 
\[
\sup_{x\in\mathcal{X}}\left\Vert w_{m}^{x}\right\Vert _{L^{2}[0,1]}^{2}=\sup_{x\in\mathcal{X}}\sum_{\ell=0}^{m-1}\left\Vert Q_{\ell}\right\Vert _{L^{2}[0,1]}^{2}Q_{\ell}(x)^{2}=\sup_{x\in[0,1]}\sum_{j=0}^{m-1}Q_{j}(x)^{2}\le\sum_{j=0}^{m-1}2j+1=m^{2},
\]
which is the assertion.
\end{proof}

\subsection{Multivariate Extension\label{subsec:Multivariate-Extension}}

The moment property \eqref{eq:MomProp} as well as the associated
norm bound \eqref{eq:momfuncnorm}, heavily rely on the proposed setting,
i.e., the design space $[0,1]$ equipped with the uniform distribution.
These assumptions are quite restrictive and need to be relaxed for
situations of practical application. To this end, we investigate a
more general setting throughout this section.

We consider the multivariate design space $\mathcal{X}=[0,1]^{d}$
with some underlying design measure $P$. This measure has a strictly
positive density $p$, such that
\[
0<p_{\min}\coloneqq\inf_{z\in\mathcal{X}}p(z)\le p_{\max}\coloneqq\sup_{z\in\mathcal{X}}p(z)<\infty,
\]
giving rise to the inner product 
\[
\left\langle f,g\right\rangle _{L^{2}(P)}=\int_{[0,1]^{d}}f(z)g(z)p(z)dz.
\]

In what follows we specify a multivariate version of $w_{m}^{x}$,
which is adjusted to this more general setting. Building on the univariate
moment property of the $w_{m}^{x}$, we demonstrate that the corresponding
tensor product satisfies a multivariate version of~\eqref{eq:MomProp}.
The next proposition reveals the precise statement. 
\begin{prop}[Upper bound of the weight function in higher dimensions]
\label{prop:MomDim} Let $x=(x_{1},\dots,x_{d})\in[0,1]^{d}$ and
consider the function 

\begin{equation}
W_{m}^{x}(z_{1},\dots,z_{d})\coloneqq\left(\prod_{i=1}^{d}w_{m}^{x_{i}}(z_{i})\right)p(z_{1},\dots,z_{n})^{-1},\label{eq:PointApprox}
\end{equation}
where $w_{m}^{x_{i}}$ are the functions defined in~\eqref{eq:wmx}.
The function $W_{m}^{x}$ satisfies the general moment property 
\begin{equation}
\int_{[0,1]^{d}}\left(\left\Vert y-z\right\Vert _{2}^{2}\right)^{\ell}W_{m}^{x}(z)p(z)dz=\left(\left\Vert y-x\right\Vert _{2}^{2}\right)^{\ell},\label{eq:MomProbGen}
\end{equation}
for all integers $\ell\le\frac{m-1}{2}$. Its norm is bounded by 
\begin{align}
\sup_{x\in\mathcal{X}}\left\Vert W_{m}^{x}\right\Vert _{L^{2}(P)}^{2} & \le c_{p}m^{2d}\label{eq:MoreDimNormBound}
\end{align}
where $c_{p}=\sup_{z\in[0,1]^{d}}p(z)^{-1}$ . 
\end{prop}

\begin{proof}
The moment property \eqref{eq:MomProbGen} follows from
\begin{align*}
\int_{[0,1]^{d}}\left(\left\Vert y-z\right\Vert _{2}^{2}\right)^{\ell}W_{m}^{x}(z)p(z)dz & =\int_{0}^{1}\dots\int_{0}^{1}\left(\sum_{i=1}^{d}(y_{i}-z_{i})^{2}\right)^{\ell}\prod_{i=1}^{d}w_{m}^{x_{i}}(z_{i})dz_{1}\dots dz_{d}\\
 & =\sum_{h_{1}+\dots+h_{d}=\ell}\binom{\ell}{h_{1},\dots,h_{d}}\prod_{i=1}^{d}\int_{0}^{1}(y_{i}-z_{i})^{2h_{i}}w_{m}^{x_{i}}(z_{i})dz_{i}\\
 & =\sum_{h_{1}+\dots+h_{d}=\ell}\binom{\ell}{h_{1},\dots,h_{d}}\prod_{i=1}^{d}(y_{i}-x_{i})^{2h_{i}}\\
 & =\left(\left\Vert y-x\right\Vert _{2}^{2}\right)^{\ell}
\end{align*}
as $\int_{0}^{1}z_{i}^{h_{i}}w_{m}^{x_{i}}(z_{i})dz_{i}=x_{i}^{h_{i}}$
holds for all integers $h_{i}\le\frac{m-1}{2}$. This is the first
assertion. 

For the second assertion recall that $\left\Vert w_{m}^{x_{i}}\right\Vert _{L^{2}[0,1]}^{2}\le m^{2}$
a holds by \eqref{eq:momfuncnorm}. Hence, we get that
\begin{align*}
\sup_{x\in\mathcal{X}}\int_{[0,1]^{d}}(W_{m}^{x}(z))^{2}p(z)dz & =\int_{0}^{1}\dots\int_{0}^{1}\left(\prod_{i=1}^{n}w_{m}^{x_{i}}(z_{i})\right)^{2}p^{-1}(z_{1},\dots,z_{n})dz_{1}\dots dz_{d}\\
 & \le\sup_{z\in[0,1]^{d}}\left|p^{-1}(z)\right|\prod_{i=1}^{d}\int_{0}^{1}(w_{m}^{x_{i}}(z_{i}))^{2}dz_{i}\le c_{p}m^{2d},
\end{align*}
which is the assertion.
\end{proof}
\begin{rem}
\label{rem:Gen2}With the same arguments as above, the function 
\[
W_{m,\prod_{i}[a_{i},b_{i}]}^{x}(z_{1},\dots,z_{d})\coloneqq\left(\prod_{i=1}^{d}w_{[a_{i},b_{i}]}^{x_{i}}(z_{i})\right)p(z_{1},\dots,z_{d})
\]
 satisfies the multivariate moment property on the hyper-rectangle
$[a_{1},b_{1}]\times\dots\times[a_{n},b_{n}]$ ( for the definition
of $w_{[a,b]}^{x_{i}}$ see Remark~\ref{rem:Gen1}). 
\end{rem}

The formula \eqref{eq:MomProbGen} is crucial for analyzing the approximation
error between $L_{k}w$ and $k_{x}$. To see this, assume for the
moment that $k(\cdot,x)$ has the series representation $k(\cdot,x)=\sum_{\ell=0}^{\infty}a_{\ell}\left\Vert x-\cdot\right\Vert _{2}^{2\ell}$.
Then, it follows from \eqref{eq:MomProbGen} that
\begin{equation}
k(x,\cdot)\approx\sum_{\ell=0}^{\left\lfloor \frac{m}{2}\right\rfloor }a_{\ell}\left\Vert x-\cdot\right\Vert _{2}^{2\ell}=\int_{\mathcal{X}}\sum_{\ell=0}^{\left\lfloor \frac{m}{2}\right\rfloor }a_{\ell}\left\Vert z-\cdot\right\Vert _{2}^{2\ell}W_{m}^{x}(z)dz\approx(L_{k}W_{m}^{x})(\cdot),\label{eq:KernelIdea}
\end{equation}
by replacing the kernel with its truncated power series. Thus, the
magnitude of the series remainder $\sum_{\ell=\left\lfloor \frac{m}{2}\right\rfloor +1}^{\infty}a_{\ell}\left\Vert x-\cdot\right\Vert _{2}^{2\ell}$
determines the approximation error between both functions. The next
section picks up this idea in mathematically precise way. 

\subsection{Approximation error}

We now elaborate the norm of the residuals $L_{k}w_{m}^{x}-k_{x}$
for a large class of kernel functions $k$. The main result is an
explicit bound of the approximation error $\left\Vert L_{k}w_{m}^{x}-k_{x}\right\Vert _{k}^{2}$,
which is uniform in the location $x\in\mathcal{X}$. To this end,
we provide a bound in the uniform norm first, which we then extend
towards $\left\Vert \cdot\right\Vert _{k}$. Throughout this section,
we consider radial kernels 
\begin{equation}
k(x,y)=\sum_{\ell=0}^{\infty}a_{\ell}Q_{\ell}\left(d^{-1}\left\Vert x-y\right\Vert _{2}^{2}\right),\label{eq:KernelLimDiff}
\end{equation}
where $\left\{ Q_{\ell}\right\} _{\ell=0}^{\infty}$ denote the scaled
Legendre polynomials introduced in \eqref{eq:LegendreDerivativeRep}.
In other words, the kernel has the form $k(x,y)=\phi\big(d^{-1}\left\Vert x-y\right\Vert _{2}^{2}\big)$,
where $\phi$ has a Legendre expansion with coefficients $(a_{\ell})_{\ell=0}^{\infty}$. 

In what follows we utilize this structure and employ a similar technique
as in \eqref{eq:KernelIdea}. As the $Q_{\ell}$ is a polynomial of
the degree at most $\ell$, $Q_{\ell}\left(d^{-1}\left\Vert x-y\right\Vert _{2}^{2}\right)$
is a linear combination of powers $\left\Vert x-y\right\Vert _{2}^{2k}$
with $k\le\ell$. Hence, by the same arguments as in \eqref{eq:KernelIdea},
the first terms in the Legendre expansions of $L_{k}w_{m}^{x}$ and
$k_{x}$ coincide. With this we derive a bound of the approximation
error in terms of the series remainder, which we state next. 
\begin{thm}
\label{thm:KernelApproxInf} Let $\mathcal{X}=[0,1]^{d}$ and $k$
be a kernel $k(x,y)=\phi\left(d^{-1}\left\Vert x-y\right\Vert _{2}^{2}\right)$
admitting the expansion 
\begin{equation}
\phi(x)=\sum_{\ell=0}^{\infty}a_{\ell}Q_{\ell}(x)\hspace{1em}x\in[0,1].\label{eq:LegExp}
\end{equation}
It holds that
\begin{align}
\sup_{x\in\mathcal{X}}\left\Vert L_{k}W_{m}^{x}-k_{x}\right\Vert _{L^{\infty}} & \le\left(1+c_{p}m^{d}\right)\mathcal{E}_{\phi}\left(\left\lfloor \frac{m-1}{2}\right\rfloor \right),\label{eq:sobolevkernelapproxinf}
\end{align}
where $\mathcal{E}\colon\mathbb{N}\to[0,\infty)$ defined by 
\begin{equation}
\mathcal{E}_{\phi}(n)\coloneqq\max_{z\in[0,1]}\left|\sum_{\ell=n+1}^{\infty}a_{\ell}Q_{\ell}(z)\right|\label{eq:Errorfunc}
\end{equation}
is the absolute maximum of the remainder of the Legendre series. 
\end{thm}

\begin{proof}
Let the kernel $k$ be such that
\[
k(x,y)=\sum_{\ell=0}^{\infty}a_{\ell}Q_{\ell}\left(d^{-1}\left\Vert x-y\right\Vert _{2}^{2}\right)
\]
and note that this expression is well defined for every $x,y\in[0,1]^{d}$,
as $d^{-1}\left\Vert x-y\right\Vert _{2}^{2}\in[0,1]$. The $\ell$th
basis function $Q_{\ell}$ is a polynomial of degree $\ell$ and it
hence follows from \eqref{eq:MomProbGen} that 
\[
\int_{\mathcal{X}}Q_{\ell}\left(d^{-1}\left\Vert y-z\right\Vert _{2}^{2}\right)W_{m}^{x}(z)p(z)dz=Q_{\ell}\left(d^{-1}\left\Vert y-x\right\Vert _{2}^{2}\right)
\]
holds for every $\ell\le\frac{m-1}{2}$. Thus, by involving the series
representation $k$, we get that
\begin{align*}
(L_{k}W_{m}^{x})(y) & =\int_{\mathcal{X}}\sum_{\ell=0}^{\infty}a_{\ell}Q_{\ell}\left(d^{-1}\left\Vert y-z\right\Vert _{2}^{2}\right)W_{m}^{x}(z)p(z)dz\\
 & =\sum_{\ell=0}^{\left\lfloor \frac{m-1}{2}\right\rfloor }a_{\ell}Q_{\ell}\left(d^{-1}\left\Vert y-x\right\Vert _{2}^{2}\right)+\sum_{\ell=\left\lfloor \frac{m-1}{2}\right\rfloor +1}^{\infty}\int_{\mathcal{X}}a_{\ell}Q_{\ell}\left(d^{-1}\left\Vert y-z\right\Vert _{2}^{2}\right)W_{m}^{x}(z)p(z)dz
\end{align*}
for every $x,y\in[0,1]^{d}$. The residual between the $L_{k}W_{m}^{x}$
and the associated point evaluation function $k_{x}$ therefore is
\begin{align*}
\MoveEqLeft[2]\left|k(y,x)-(L_{k}W_{m}^{x})(y)\right|\\
 & =\left|k(y,x)-\sum_{\ell=0}^{\left\lfloor \frac{m-1}{2}\right\rfloor }a_{\ell}Q_{\ell}\left(d^{-1}\left\Vert y-x\right\Vert _{2}^{2}\right)+\int_{\mathcal{X}}\sum_{\ell=\left\lfloor \frac{m-1}{2}\right\rfloor +1}^{\infty}a_{\ell}Q_{\ell}\left(d^{-1}\left\Vert y-z\right\Vert _{2}^{2}\right)W_{m}^{x}(z)dz\right|\\
 & =\left|\sum_{\ell=\left\lfloor \frac{m-1}{2}\right\rfloor +1}^{\infty}a_{\ell}Q_{\ell}\left(d^{-1}\left\Vert y-x\right\Vert _{2}^{2}\right)+\int_{\mathcal{X}}\sum_{\ell=\left\lfloor \frac{m-1}{2}\right\rfloor +1}^{\infty}a_{\ell}Q_{\ell}\left(d^{-1}\left\Vert y-z\right\Vert _{2}^{2}\right)W_{m}^{x}(z)dz\right|\\
 & \le\left|\sum_{\ell=\left\lfloor \frac{m-1}{2}\right\rfloor +1}^{\infty}a_{\ell}Q_{\ell}\left(d^{-1}\left\Vert y-x\right\Vert _{2}^{2}\right)\right|+\left\Vert \sum_{\ell=\left\lfloor \frac{m-1}{2}\right\rfloor +1}^{\infty}a_{\ell}Q_{\ell}\left(d^{-1}\left\Vert y-z\right\Vert _{2}^{2}\right)\right\Vert _{L^{2}(P)}\left\Vert W_{m}^{x}\right\Vert _{L^{2}(P)},
\end{align*}
where we involved the Cauchy\textendash Schwarz inequality for the
last term. Taking the supremum with respect to $x$ and $y$ as well
as involving the bound \eqref{eq:MoreDimNormBound}, we get that 
\[
\sup_{x\in[0,1]^{d}}\left\Vert L_{k}W_{m}^{x}-k_{x}\right\Vert _{L^{\infty}}\le\mathcal{E}_{\phi}\left(\left\lfloor \frac{m-1}{2}\right\rfloor \right)+\mathcal{E}_{\phi}\left(\left\lfloor \frac{m-1}{2}\right\rfloor \right)c_{p}m^{d},
\]
which is the assertion.
\end{proof}
The uniform norm $\|\cdot\|_{\infty}$ is generally weaker than $\left\Vert \cdot\right\Vert _{k}$
and \eqref{eq:sobolevkernelapproxinf} is therefore not sufficient
to bound on the objective value \eqref{eq:centralobj}. However, for
the specific function $L_{k}W_{m}^{x}-k_{x}$, it is indeed possible
to derive a bound on the approximation error $\left\Vert L_{k}W_{m}^{x}-k_{x}\right\Vert _{k}$
by considering pointwise residuals. Due to the linearity of the inner
product, the approximation error has the explicit representation
\begin{align*}
\left\Vert L_{k}W_{m}^{x}-k_{x}\right\Vert _{k}^{2} & =\left\langle L_{k}W_{m}^{x},L_{k}W_{m}^{x}\right\rangle _{k}-2\left\langle L_{k}W_{m}^{x},k_{x}\right\rangle _{k}+\left\langle k_{x},k_{x}\right\rangle _{k}\\
 & =\left\langle L_{k}W_{m}^{x},L_{k}W_{m}^{x}-k_{x}\right\rangle _{k}-\left\langle L_{k}W_{m}^{x},k_{x}\right\rangle _{k}+\left\langle k_{x},k_{x}\right\rangle _{k}\\
 & =\left\langle W_{m}^{x},L_{k}W_{m}^{x}-k_{x}\right\rangle _{L^{2}(P)}-\left\langle L_{k}W_{m}^{x},k_{x}\right\rangle _{k}+\left\langle k_{x},k_{x}\right\rangle _{k},
\end{align*}
by employing \eqref{eq:normid}. Hence, it follows the reproducing
property of $k_{x}$ as well as the Cauchy\textendash Schwarz inequality
that 
\begin{align*}
\left\langle W_{m}^{x},L_{k}W_{m}^{x}-k_{x}\right\rangle _{L^{2}(P)}-\left\langle L_{k}W_{m}^{x},k_{x}\right\rangle _{k}+\left\langle k_{x},k_{x}\right\rangle  & \le\left\Vert W_{m}^{x}\right\Vert _{L^{2}}\left\Vert L_{k}W_{m}^{x}-k_{x}\right\Vert _{L^{2}(P)}-L_{k}W_{m}^{x}(x)+k_{x}(x)\\
 & \le\left\Vert W_{m}^{x}\right\Vert _{L^{2}}\left\Vert L_{k}W_{m}^{x}-k_{x}\right\Vert _{L^{\infty}}+\left\Vert L_{k}W_{m}^{x}-k_{x}\right\Vert _{L^{\infty}}.
\end{align*}
This gives the bound 
\begin{equation}
\left\Vert L_{k}W_{m}^{x}-k_{x}\right\Vert _{k}^{2}\le\left\Vert W_{m}^{x}\right\Vert _{L^{2}(P)}\left\Vert L_{k}W_{m}^{x}-k_{x}\right\Vert _{L^{\infty}}+\left\Vert L_{k}W_{m}^{x}-k_{x}\right\Vert _{L^{\infty}},\label{eq:HkUniform}
\end{equation}
which only depends on the uniform error $\left\Vert L_{k}W_{m}^{x}-k_{x}\right\Vert _{L^{\infty}}$
as well as the norm of the associated weight function $W_{m}^{x}$.
In the next Theorem, we pick up this idea and establish an explicit
bound on the approximation error $\left\Vert L_{k}W_{m}^{x}-k_{x}\right\Vert _{k}^{2}$.
\begin{thm}
\label{thm:KernelApproxHk} Under the same assumptions as in Theorem~\ref{thm:KernelApproxInf},
the inequality

\begin{equation}
\sup_{x\in[0,1]^{d}}\left\Vert L_{k}W_{m}^{x}-k_{x}\right\Vert _{k}^{2}\le4c_{p}^{2}m^{2d}\mathcal{E}_{\phi}\left(\left\lfloor \frac{m-1}{2}\right\rfloor \right)\label{eq:sobolevkernelapproxHk}
\end{equation}
holds for every $m\in\mathbb{N}$.
\end{thm}

\begin{proof}
As a consequence of \eqref{eq:HkUniform} and \eqref{eq:MoreDimNormBound}
we have that 
\begin{align*}
\left\Vert L_{k}W_{m}^{x}-k_{x}\right\Vert _{k}^{2} & \le\left\Vert W_{m}^{x}\right\Vert _{L^{2}(P)}\left\Vert L_{k}W_{m}^{x}-k_{x}\right\Vert _{L^{\infty}}+\left\Vert L_{k}W_{m}^{x}-k_{x}\right\Vert _{L^{\infty}}\\
 & \le\big(1+c_{p}m^{d}\big)\left\Vert L_{k}W_{m}^{x}-k_{x}\right\Vert _{L^{\infty}}
\end{align*}
holds true. Employing the error bound \eqref{eq:sobolevkernelapproxinf},
we further get that
\begin{align*}
\left(1+c_{p}m^{d}\right)\left\Vert L_{k}W_{m}^{x}-k_{x}\right\Vert _{L^{\infty}} & \le\left(1+c_{p}m^{d}\right)^{2}\mathcal{E}_{\phi}\left(\left\lfloor \frac{m-1}{2}\right\rfloor \right)\\
 & \le4c_{p}^{2}m^{2d}\mathcal{E}_{\phi}\left(\left\lfloor \frac{m-1}{2}\right\rfloor \right)
\end{align*}
as $c_{p},m^{d}\ge1$. This is the assertion. 
\end{proof}

\subsection{Regularized approximation of point evaluation functions}

Thus far we studied the approximation error $\left\Vert L_{k}W_{m}^{x}-k_{x}\right\Vert $
as well as the norm of the underlying weight function $\left\Vert W_{m}^{x}\right\Vert _{L^{2}(P)}$
separately. This section connects both quantities to derive a bound
on 

\begin{equation}
\sup_{x\in\mathcal{X}}\inf_{w\in L^{2}(P)}\lambda\left\Vert w\right\Vert _{L^{2}(P)}^{2}+\left\Vert L_{k}w-k_{x}\right\Vert _{k}^{2}\label{eq:WorstCase}
\end{equation}
and thus on the worst case of \eqref{eq:centralobj}. The bound is
a function of the regularization parameter $\lambda$, depending on
the decay rate of the error function $\mathcal{E}_{\phi}$ (cf.~\eqref{eq:Errorfunc}).
This decay is either polynomial or exponential, i.e., the are positive
real numbers $s,\rho,C_{\phi,s}$ and $C_{\phi,\rho}$ such that
\begin{align}
\mathcal{E}_{\phi}(n) & \le C_{\phi,s}n^{-s}\hspace{1em}\text{or}\label{eq:errorpoly}\\
\mathcal{E}_{\phi}(n) & \le C_{\phi,\rho}e^{-\rho n}\label{eq:errorexp}
\end{align}
holds for all $n\ge n_{0}$. In the following theorem we establish
the bounds corresponding to the cases \eqref{eq:errorpoly} and \eqref{eq:errorexp}. 
\begin{thm}[Bound of the objective value]
\label{thm:CentralBounds} Suppose that the assumptions of Theorem~\ref{thm:KernelApproxInf}
are satisfied. For the error functions as in \eqref{eq:errorpoly}
and \eqref{eq:errorexp} as well as sufficiently small $\lambda$,
we have the bounds 
\begin{align}
\sup_{x\in\mathcal{X}}\inf_{w\in L^{2}}\lambda\left\Vert w\right\Vert _{L^{2}(P)}^{2}+\left\Vert L_{k}w-k_{x}\right\Vert _{k}^{2} & \le\lambda C_{s}\lambda^{-\frac{2d}{s}}\hspace{1em}\text{and}\label{eq:OptValPoly}\\
\sup_{x\in\mathcal{X}}\inf_{w\in L^{2}}\lambda\left\Vert w\right\Vert _{L^{2}(P)}^{2}+\left\Vert L_{k}w-k_{x}\right\Vert _{k}^{2} & \le\lambda C_{\rho}\ln(\lambda)^{2d},\label{eq:OptValExp}
\end{align}
respectively. Here, $C_{s}$ and $C_{\rho}$ are positive constants.
\end{thm}

\begin{proof}
For the first case assume that \eqref{eq:errorpoly} holds true and
consider the function $s(\lambda)=\lambda^{-\frac{1}{s}}$. Employing
the bounds \eqref{eq:MoreDimNormBound} and \eqref{eq:sobolevkernelapproxHk},
it follows that 
\begin{align*}
\inf_{w\in L^{2}}\lambda\left\Vert w\right\Vert _{L^{2}}^{2}+\left\Vert L_{k}w-k_{x}\right\Vert _{k}^{2} & \le\lambda\left\Vert W_{m}^{x}\right\Vert _{L^{2}}^{2}+\left\Vert L_{k}W_{m}^{x}-k_{x}\right\Vert _{k}^{2}\\
 & \le\lambda c_{p}^{2}m^{2d}+4c_{p}^{2}m^{2d}\mathcal{E}_{\phi}\left(\left\lfloor \frac{m-1}{2}\right\rfloor \right).
\end{align*}
Setting $m=\left\lceil s(\lambda)\right\rceil $ and we get for sufficiently
small $\lambda$ (e.g. $\lambda\le(2n_{0}+1)^{-s}$), we further get
that 
\begin{align*}
\lambda c_{p}^{2}m^{2d}+2c_{p}^{2}m^{2d}\mathcal{E}_{\infty}\left(\left\lfloor \frac{m-1}{2}\right\rfloor \right) & =\lambda c_{p}^{2}\left\lceil s_{p}(\lambda)\right\rceil ^{2d}+4c_{p}^{2}\left\lceil s_{p}(\lambda)\right\rceil ^{2d}\mathcal{E}_{\phi}\left(\left\lfloor \frac{\left\lceil s(\lambda)\right\rceil -1}{2}\right\rfloor \right)\\
 & \le\lambda c_{p}^{2}(\lambda^{-\frac{1}{s}}+1)^{2d}+4c_{p}^{2}\left\lceil (\lambda^{-\frac{1}{s}}+1)^{2d}\right\rceil ^{2d}\mathcal{E}_{\phi}\left(\left\lfloor \frac{\left\lceil \lambda^{-\frac{1}{s}}\right\rceil -1}{2}\right\rfloor \right)\\
 & \le\lambda c_{p}^{2}2^{2d}\lambda^{-\frac{2d}{s}}+4c_{p}^{2}2^{2d}\lambda^{-\frac{2d}{s}}C_{\phi,s}\left\lfloor \frac{\left\lceil \lambda^{-\frac{1}{s}}\right\rceil -1}{2}\right\rfloor ^{-s}
\end{align*}
by \eqref{eq:errorpoly}. Now note that $\left\lfloor \frac{\left\lceil \lambda^{-\frac{1}{s}}\right\rceil -1}{2}\right\rfloor \le\frac{\lambda^{-\frac{1}{s}}-3}{2}$
as well as $\lambda^{-\frac{1}{s}}-3\le4^{-1}\lambda^{-\frac{1}{s}}$
holds for every $\lambda\in(0,4^{-s})$. Thus, have that
\begin{align*}
\lambda c_{p}^{2}2^{2d}\lambda^{-\frac{2d}{s}}+4c_{p}^{2}2^{2d}\lambda^{-\frac{2d}{s}}C_{\phi,s}\left\lfloor \frac{\left\lceil \lambda^{-\frac{1}{s}}\right\rceil -1}{2}\right\rfloor ^{-s} & \le\lambda c_{p}^{2}2^{2d}\lambda^{-\frac{2d}{s}}+4c_{p}^{2}2^{2d}\lambda^{-\frac{2d}{s}}C_{\phi,s}2^{s}\left(\lambda^{-\frac{1}{s}}-3\right)^{-s}\\
 & \le\lambda c_{p}^{2}2^{2d}\lambda^{-\frac{2d}{s}}+4c_{p}^{2}2^{2d}\lambda^{-\frac{2d}{s}}C_{\phi,s}8^{s}\lambda^{-\frac{1}{s}}\\
 & \le\lambda\left(c_{p}^{2}2^{2d}+4c_{p}^{2}2^{2d}C_{\phi,s}8^{s}\right)\lambda^{-\frac{2d}{s}}\\
 & =\lambda C_{s}\lambda^{-\frac{2d}{s}}
\end{align*}
where $C_{s}\coloneqq\left(c_{p}^{2}2^{2d}+4c_{p}^{2}2^{2d}C_{\phi,s}8^{s}\right)$.
This is the first assertion. The second follows along the same lines,
by considering the function $s(\lambda)=\ln(\lambda^{-1})$.
\end{proof}
\begin{rem}
The decay assumptions \eqref{eq:errorpoly} and \eqref{eq:errorexp}
do not depict all possible scenarios of $\mathcal{E}_{\phi}$. For
example for the Gaussian kernel, the error function decays super exponentially,
i.e., it holds 
\[
\mathcal{E}_{\phi}(n)\le Ce^{-\rho n\ln n}
\]
for some $\rho>0$. However, these cases provide only minor improvements
to the final results and are therefore not considered throughout this
work. 
\end{rem}

\subsection{Proof of the main result}

In what follows we provide the proof of the main Theorem~\ref{thm:Main}.
To this end, we establish the mentioned characterization \eqref{eq:centralobj}
of the maximal marginal degrees of freedom first. Subsequently, we
address the error function $\mathcal{E}_{\phi}$, where we link the
decay assumptions \eqref{eq:poldec} and \eqref{eq:expdec} to the
assumptions \ref{asuA} and \ref{asuB}. Finally, we connect Theorem~\ref{thm:CentralBounds}
with the identity \eqref{eq:centralobj} and prove the main result.

The starting point of the characterization \eqref{eq:MMDF} is the
minimizer of \eqref{eq:centralobj}. To get an explicit representation
of this minimizer we employ the following auxiliary result for the
directional derivative. 
\begin{lem}
Consider the linear functionals $f,g\colon L^{2}\to\mathbb{R}$ with
$f(w)=\left\langle w,w\right\rangle _{L^{2}(P)}$ and $g(w)=\left\langle a,w\right\rangle _{L^{2}(P)}$
with $a\in L^{2}$. Then, the directional derivatives are
\begin{align}
\nabla_{v}f(w) & =2\left\langle w,v\right\rangle _{L^{2}(P)}\label{eq:Derivative1}\\
\nabla_{v}g(w) & =\left\langle f,v\right\rangle _{L^{2}(P)},\label{eq:Derivative2}
\end{align}
where $v\in L^{2}(P)$.
\end{lem}

\begin{proof}
For $h>0$ and $v\in L^{2}(P)$ it holds that 
\[
f(w+hv)-f(w)=\left\langle w+hv,w+hv\right\rangle _{L^{2}(P)}-\left\langle w,w\right\rangle _{L^{2}(P)}=h^{2}\left\langle v,v\right\rangle +2h\left\langle w,v\right\rangle 
\]
by utilizing the linearity of the $L^{2}(P)$ inner product. Thus,
the directional derivative is 
\[
\lim_{h\to0}\frac{f(w+hv)-f(w)}{h}=\lim_{h\to0}h\left\langle v,v\right\rangle _{L^{2}(P)}+2\left\langle w,v\right\rangle _{L^{2}(P)}=2\left\langle w,v\right\rangle _{L^{2}(P)}.
\]
and thus \eqref{eq:Derivative1}. \eqref{eq:Derivative2} follows
by the same approach. 
\end{proof}
\begin{lem}
\label{lem:centralid} The function $w_{\lambda}^{x}=(\lambda+L_{k})^{-1}k_{x}$
minimizes the objective \eqref{eq:centralobj}. Moreover, the identity
\begin{equation}
\mathcal{N}_{\infty}(\lambda)=\sup_{x\in\mathcal{X}}\inf_{w\in L^{2}(P)}\left\Vert w\right\Vert _{L^{2}(P)}^{2}+\lambda^{-1}\left\Vert L_{k}w-k_{x}\right\Vert _{k}^{2}\label{eq:OptWeightFunc}
\end{equation}
holds for every $\lambda>0$. 
\end{lem}

\begin{proof}
Expanding the objective \eqref{eq:centralobj} in terms of inner products,
we have that
\[
\lambda\left\Vert w\right\Vert _{L^{2}(P)}^{2}+\left\Vert L_{k}w-k_{x}\right\Vert _{k}^{2}=\lambda\left\langle w,w\right\rangle _{L^{2}(P)}+\left\langle L_{k}w,w\right\rangle _{L^{2}(P)}-2\left\langle w,k_{x}\right\rangle _{L^{2}(P)}+\left\langle k_{x},k_{x}\right\rangle _{k},
\]
by employing the relationship \ref{eq:normid}. This objective is
convex and a point $w$ therefore is a minimizer of \eqref{eq:centralobj},
if the directional derivative vanishes for every direction $v\in L^{2}(P)$.
By \eqref{eq:Derivative1} and \eqref{eq:Derivative2}, this condition
is satisfied for 
\[
2\lambda w+2L_{k}w-2k_{x}=0,
\]
which holds for
\begin{equation}
w_{\lambda}^{x}=(\lambda+L_{k})^{-1}k_{x}.\label{eq:weightoptfunc}
\end{equation}
Thus, we have that 
\[
\inf_{w\in L^{2}}\lambda\left\Vert w\right\Vert _{L^{2}(P)}^{2}+\left\Vert L_{k}w-k_{x}\right\Vert _{k}^{2}=\lambda\left\Vert w_{\lambda}^{x}\right\Vert _{L^{2}(P)}^{2}+\left\Vert L_{k}w_{\lambda}^{x}-k_{x}\right\Vert _{k}^{2}
\]
for every $x\in\mathcal{X}$, which is the first assertion.

For the second observe the residual relation $k_{x}-L_{k}w_{\lambda}^{x}=(\lambda+L_{k})w_{\lambda}^{x}-L_{k}w_{\lambda}^{x}=\lambda w_{\lambda}^{x}$.
Hence, the objective is 
\begin{align*}
\lambda\left\Vert w_{\lambda}^{x}\right\Vert _{L^{2}(P)}^{2}+\left\Vert L_{k}w_{\lambda}^{x}-k_{x}\right\Vert _{k}^{2} & =\lambda\left\Vert w_{\lambda}^{x}\right\Vert _{L^{2}(P)}^{2}+\lambda^{2}\left\Vert w_{\lambda}^{x}\right\Vert _{k}^{2}\\
 & =\lambda\left\langle w_{\lambda}^{x},L_{k}w_{\lambda}^{x}\right\rangle _{k}+\lambda\left\langle w_{\lambda}^{x},\lambda w_{\lambda}^{x}\right\rangle _{k}\\
 & =\lambda\left\langle w_{\lambda}^{x},k_{x}\right\rangle _{k}
\end{align*}
as $(\lambda+L_{k})w_{\lambda}=k_{x}$. Taking the supremum in $x$
it follows that
\[
\sup_{x\in\mathcal{X}}\inf_{w\in L^{2}}\lambda\left\Vert w_{\lambda}^{x}\right\Vert _{L^{2}(P)}^{2}+\left\Vert L_{k}w_{\lambda}^{x}-k_{x}\right\Vert _{k}^{2}=\sup_{x\in\mathcal{X}}\lambda\left\langle w_{\lambda}^{x},k_{x}\right\rangle _{k}=\lambda\mathcal{N}_{\infty}(\lambda).
\]
Multiplying both side with $\lambda^{-1}$ provides the assertion. 

The characterization \eqref{eq:OptWeightFunc} as well as the bounds
in Theorem~\ref{thm:CentralBounds} are already sufficient to prove
Theorem~\ref{thm:Main}. However, the underlying assumptions of Theorem~\ref{thm:CentralBounds}
and Theorem~\ref{thm:Main} are quite different. While the former
considers smoothness properties of the outer function, the latter
considers the decay of the associated series remainder. In what follows
we connect these concepts, basic theory of approximating function
by orthogonal polynomials 
\end{proof}
\begin{prop}
Let $\phi\colon[0,1]\to\mathbb{R}$ with expansion $\phi(x)=\sum_{\ell=0}^{\infty}c_{\ell}Q_{\ell}(x)$.
If Assumption~\ref{asuA} is satisfied, it holds for $n\ge s+1$
that 
\begin{equation}
|c_{n}|\le\frac{\sqrt{2}V_{s}}{\sqrt{2n+1}\sqrt{\pi(2n-2s-1)}}\prod_{k=1}^{s}h_{n-k},\label{eq:poldec}
\end{equation}
where $h_{n}\coloneqq\left(n+\frac{1}{2}\right)^{-1}$ and the product
is assumed to be $0$ if $s=0$.

If Assumption \ref{asuB} is satisfied, then for each $k\ge0$ 
\begin{equation}
|c_{0}|\le\frac{\sqrt{2}D(\rho)}{2},\hspace{1em}|c_{n}|\le D(\rho)\frac{\sqrt{2}n^{\frac{1}{2}}}{\sqrt{2n+1}\rho^{n}},\hspace{1em}n\ge1,\label{eq:expdec}
\end{equation}
where $D(\rho)$ is defined by 
\[
D(\rho)\coloneqq\frac{2L(\mathcal{E}_{\rho})}{\pi\sqrt{\rho^{2}-1}}\max_{z\in\mathcal{E}_{\rho}}|\tilde{\phi}(z)|.
\]
Here $L(\mathcal{E}_{\rho})$ denotes the length of the circumference
of $\mathcal{E}_{\rho}$. 
\end{prop}

\begin{proof}
Let $\phi\colon[0,1]\to\mathbb{R}$ with an expansion
\[
\phi(x)=\sum_{\ell=0}^{\infty}c_{\ell}Q_{\ell}(x).
\]
in terms of the scaled Legendre polynomials. Further let $\tilde{\phi}\colon[-1,1]\to\mathbb{R}$
be defined as $\tilde{\phi}(x)=\phi\left(\frac{x+1}{2}\right)$. Employing
the identity $Q_{\ell}\left(\frac{x+1}{2}\right)=P_{\ell}(x)$ from
\eqref{eq:LegendreDerivativeRep}, this function has the representation
\begin{align*}
\tilde{\phi}(x) & =\sum_{\ell=0}^{\infty}c_{\ell}Q_{\ell}\left(\frac{x+1}{2}\right)\\
 & =\sum_{\ell=0}^{\infty}\sqrt{\frac{2\ell+1}{2}}c_{\ell}P_{\ell}(x)=\sum_{\ell=0}^{\infty}a_{\ell}P_{\ell}(x).
\end{align*}
with $a_{\ell}=\sqrt{\frac{2\ell+1}{2}}c_{\ell}$. Hence, the expansion
of $\tilde{f}$ in terms of the usual Legendre polynomials has coefficients
$a_{\ell}$. By \citet[Theorem 2.2]{Wang2018} and \citet[Lemma 2]{Wang2021}
the coefficients satisfy the bound
\[
|a_{n}|\le\frac{V_{s}}{\sqrt{\pi(2n-2s-1)}}\prod_{k=1}^{s}h_{n-k}
\]
if Assumption~\ref{asuA} holds as well as 
\[
|a_{0}|\le\frac{D(\rho)}{2},\hspace{1em}|a_{n}|\le D(\rho)\frac{n^{\frac{1}{2}}}{\rho^{n}}
\]
if Assumption~\ref{asuB} holds. The assertion is now immediate by
the identity $a_{\ell}=\sqrt{\frac{2\ell+1}{2}}c_{\ell}$. 
\end{proof}
The term on the right hand side of \eqref{eq:poldec} is slightly
inconvenient, as it involves a product of different terms, rather
than a power of $n$. However, for $n\ge2s+1$ we may simplify this
term. Indeed, for $n\ge2s+1$, the inequality $\frac{n}{2}\le n-k+\frac{1}{2}$
holds true for $k\le s$ and therefore 
\begin{align}
\frac{\sqrt{2}V_{s}}{\sqrt{\pi(2n-2s-1)}}\prod_{k=1}^{s}h_{n-k} & =\frac{\sqrt{2}V_{s}}{\sqrt{\pi(2n-2s-1)}}\prod_{k=1}^{s}\left(n-k+\frac{1}{2}\right)^{-1}\le\frac{\sqrt{2}V_{s}}{\text{\ensuremath{\sqrt{\pi n}}}}\prod_{k=1}^{s}\frac{2}{n}=C_{V}n^{-(s+\frac{1}{2})}\label{eq:ConvinientBound}
\end{align}
 with $C_{V}\coloneqq\sqrt{2}2^{s}V_{s}\pi^{-\frac{1}{2}}$. Building
on this bound we now prove the main results of Theorem~\ref{thm:Main}.
To this end, we connect the results of this section with Theorem~\eqref{thm:CentralBounds}.
\begin{proof}[Proof of Theorem~\ref{thm:Main}]
 If \ref{asuA} is satisfied, it follows from \eqref{eq:poldec}
and \eqref{eq:ConvinientBound} that 
\[
\left|c_{\ell}\right|\le\frac{1}{\sqrt{2\ell+1}}C_{V}\ell^{-(s+\frac{1}{2})}
\]
holds for all $\ell\ge2s+1$. Therefore, as $\max_{z\in[0,1]}\left|\frac{1}{\sqrt{2\ell+1}}Q_{\ell}(z)\right|\le1$,
we have for the error function that 
\begin{align*}
\mathcal{E}_{\phi}(n) & \le\max_{z\in[0,1]}\left|\sum_{\ell=n+1}^{\infty}c_{\ell}Q_{\ell}(z)\right|\le\sum_{\ell=n+1}^{\infty}C_{V}n^{-(s+\frac{1}{2})}
\end{align*}
whenever $n\ge2s+1$. For the right hand side it follows that
\[
\sum_{\ell=n+1}^{\infty}C_{V}n^{-(s+\frac{1}{2})}\le C_{V}\int_{n}^{\infty}x^{-(s+\frac{1}{2})}dx=(s+\frac{1}{2})\cdot n^{-s+\frac{1}{2}}
\]
and thus $\mathcal{E}_{\phi}(n)$ has the form \eqref{eq:errorpoly}.
Hence, we get from \eqref{eq:OptWeightFunc} as well as \eqref{eq:OptValPoly}
that 
\[
\mathcal{N}_{\infty}(\lambda)=\sup_{x\in\mathcal{X}}\inf_{w\in L^{2}(P)}\left\Vert w\right\Vert _{L^{2}}^{2}+\lambda^{-1}\left\Vert L_{k}w-k_{x}\right\Vert _{k}^{2}\le C_{s}\lambda^{-\frac{2d}{s-\frac{1}{2}-d}}
\]
and thus the first assertion. For the second, we get from the same
\[
\mathcal{E}_{\phi}(n)\le\sum_{\ell=n+1}^{\infty}\ell^{\frac{1}{2}}e^{-\rho\ell}\le\int_{n}^{\infty}x^{\frac{1}{2}}e^{-\rho x}dx\le\int_{n}^{\infty}e^{-\frac{\rho}{2}x}dx=\frac{2}{\rho}e^{-\frac{\rho n}{2}}
\]
holds true. Thus, employing \eqref{eq:OptWeightFunc} as well as \eqref{eq:OptValPoly}
that we have that 

\[
\mathcal{N}_{\infty}(\lambda)=\sup_{x\in\mathcal{X}}\inf_{w\in L^{2}(P)}\left\Vert w\right\Vert _{L^{2}}^{2}+\lambda^{-1}\left\Vert L_{k}w-k_{x}\right\Vert _{k}^{2}\le C_{\rho}\ln(\lambda)^{2d}.
\]
This concludes the proof.
\end{proof}
\bibliographystyle{abbrvnat}
\bibliography{PaulLibNew}

\end{document}